\documentclass{article}
    \PassOptionsToPackage{numbers, compress}{natbib}
    \usepackage[final]{neurips_2019}

\usepackage{hyperref}
\usepackage{cleveref}
\usepackage[utf8]{inputenc} % allow utf-8 input
\usepackage[T1]{fontenc}    % use 8-bit T1 fonts
\usepackage{hyperref}       % hyperlinks
\usepackage{url}            % simple URL typesetting
\usepackage{booktabs}       % professional-quality tables
\usepackage{amsfonts}       % blackboard math symbols
\usepackage{nicefrac}       % compact symbols for 1/2, etc.
\usepackage{microtype}      % microtypography
\usepackage{multirow}
\usepackage{amsthm}
\usepackage{amsmath,bm}
\usepackage[english]{babel}

\newtheorem{theorem}{Theorem}[section]
\newtheorem{corollary}{Corollary}

\newtheorem{definition}{Definition}
\usepackage{booktabs}
\usepackage{indentfirst}
\usepackage{graphicx}

\usepackage{listings}
\usepackage{color}

\definecolor{dkgreen}{rgb}{0,0.6,0}
\definecolor{gray}{rgb}{0.5,0.5,0.5}
\definecolor{mauve}{rgb}{0.58,0,0.82}

\title{A Tensorized Transformer for Language Modeling}
\author{%
  {Xindian Ma}\textsuperscript{\rm 1}, 
  {Peng Zhang}{\textsuperscript{\rm 1}}{\thanks{Corresponding Author: Peng Zhang}}~~, 
  {Shuai Zhang}\textsuperscript{\rm 1},\\
  {\textbf{Nan Duan}}\textsuperscript{\rm 2}\textbf{,} 
  {\textbf{Yuexian Hou}}\textsuperscript{\rm 1}\textbf{,}  
  {\textbf{Dawei Song}}\textsuperscript{\rm 3}\textbf{,}  
  {\textbf{Ming Zhou}}\textsuperscript{\rm 2}\\
  \textsuperscript{\rm 1}College of Intelligence and Computing, Tianjin University, Tianjin, China\\
  \textsuperscript{\rm 2}Microsoft Research Asia, Beijing, China\\
  \textsuperscript{\rm 3}School of Computer Science and Technology, Beijing Institute of Technology, Beijing, China\\
  \{xindianma, pzhang, szhang96, yxhou\}@tju.edu.cn\\
  \{nanduan, mingzhou\}@microsoft.com  \\
  \{dwsong\}@bit.edu.cn
}

\begin{document}

\maketitle

\begin{abstract}
Latest development of neural models has connected the encoder and decoder through a self-attention mechanism. In particular, Transformer, which is solely based on self-attention, has led to breakthroughs in Natural Language Processing (NLP) tasks. However, the multi-head attention mechanism, as a key component of Transformer, limits the effective deployment of the model to a resource-limited setting. In this paper, based on the ideas of tensor decomposition and parameters sharing, we propose a novel self-attention model (namely Multi-linear attention) with Block-Term Tensor Decomposition (BTD). We test and verify the proposed attention method on three language modeling tasks (i.e., PTB, WikiText-103 and One-billion) and a neural machine translation task (i.e., WMT-2016 English-German). Multi-linear attention can not only largely compress the model parameters but also obtain performance improvements, compared with a number of language modeling approaches, such as Transformer, Transformer-XL, and Transformer with tensor train decomposition.
\end{abstract}

%BERT~\cite{devlin2018bert} is a bidirectional encoder representations from Transformers~\cite{vaswani2017attention}, which obtains new state-of-the-art(SoTA) results on eleven natural language processing tasks. 
%   we introduce a novel representation appoach for multi-head attention mechanism based on Block-Term Tensor Decomposition (BT),
\section{Introduction}
In NLP, Neural language model pre-training has shown to be effective for improving many tasks~\cite{devlin2018bert,peters2018deep}. Transformer~\cite{vaswani2017attention} is based solely on the attention mechanism, and dispensing with recurrent and convolutional networks entirely. At present, this model has received extensive attentions and plays an key role in many neural language models, such as BERT~\cite{devlin2018bert}, GPT~\cite{radford2018improving} and Universal Transformer~\cite{dehghani2018universal}.
However, in Transformer based model, a lot of model parameters may cause problems in training and deploying these parameters in a resource-limited setting. Thus, the compression of large neural pre-training language models has been an essential problem in NLP research. 
%Embedding layers map the input words into continuous representations, and usually has the form of look-up tables. 
% Specifically, the standard embedding can be replaced by TT-embedding with the higher compression ratio. 

In literature, there are some compression methods~\cite{khrulkov2019tensorized,ye2018learning,han2015learning} proposed. 
When the vocabulary is large, the corresponding weight matrices can be enormous. 
Tensorized embedding (TE)~\cite{khrulkov2019tensorized} uses the tensor-train~\cite{oseledets2011tensor} to compress the embedding layers in Transformer-XL~\cite{dai2019transformer}, but has not compressed the attention layer.
% In TE~\cite{khrulkov2019tensorized}, researchers only study the compression of input embedding layers, rathar than the attention layer.
Recently, Block-Term Tensor Decomposition(BTD)~\cite{de2008decompositions} is used to compress recurrent neural networks (RNNs)~\cite{ye2018learning}.
Ye \textit{et al.} \cite{ye2018learning} propose a compact flexible structure to deal with the large number of model parameters instead by high dimensional inputs in training recurrent neural networks (RNNs). 
This method greatly reduces the parameters of RNNs and improves their training efficiency. Still, the model only considers the input layer compression by the idea of low-rank approximation.
On the other hand, some methods~\cite{han2015learning,buci2006model} aim to develop a specific structure on its weight matrices and can reduce the parameters of the models. However, the new structure after compressing can not be integrated into the model~\cite{vaswani2017attention}.

In Transformer, the multi-head attention is a key part and it is constructed by a large number of parameters. 
Specifically, Ashish et.al~\cite{vaswani2017attention} compute the attention function on a set of queries simultaneously, packed together into a matrix $Q$, while the keys and values are also packed together into matrices $K$ and $V$, respectively. The attention function then adopts a no-linear function $softmax$ over two matrices $Q$ and $K$. There are two challenges to find a high-quality compression method to compress the multi-head attention in Transformer. 
% In this paper, we mainly focus on the compression of multi-head attention in Transformer.

First, the self-attention function in Transformer is a non-linear function, which makes it difficult to compress. In order to address this challenge, we first prove that the output of the attention function of the self-attention model~\cite{vaswani2017attention} can be linearly represented by a group of orthonormal base vectors. Then, by initializing a low rank core tensor, we use Tucker-decomposition~\cite{tucker1966some,li2017bt} to reconstruct a new attention representation, where $Q$, $K$ and $V$ can be considered as factor matrices. In order to construct the multi-head mechanism and compress the model, we use the method of Block-Term Tensor Decomposition (BTD), which is a combination of CP decomposition~\cite{carroll1970analysis} and Tucker decomposition~\cite{tucker1966some}. The difference is that three factor matrices $Q,~K$ and $V$ are shared in constructing each $3$-order block tensor. This process can reduce many parameters.
 
The second challenge is that the attention model after compressing can not be directly integrated into the encoder and decoder framework of Transformer~\cite{vaswani2017attention,dai2019transformer}.
In order to address this challenge, there are three steps as follows. First, the average of each block tensor can be computed; Second, multiple matrices can be given by tensor split. Third, the concatenation of these matrices can serve as the input to the next layer network in Transformer. After that, it can be integrated into the encoder and decoder framework of Transformer~\cite{vaswani2017attention,dai2019transformer} and trained end-to-end. 
Moreover, we also prove that the $3$-order tensor can reconstruct the scaled dot-product attention in Transformer by a sum on a particular dimension.
% In this paper, we mainly focus on the compression of multi-head attention in Transformer and introduce a new compression method for Transformer language model to address these challenges. 
% We first prove that the output of the attention function from self-attention model~\cite{vaswani2017attention} can be linearly represented by a group of orthonormal base vectors.
% $Q$, $K$ and $V$ can be consider as factor matrix. Then, we use the method of Tucker-decomposition~\cite{tucker1966some,li2017bt} to reconstruct a new attention representation by initializing a low rank core tensor. In order to construct the multi-head mechanism and compress the model, we use the method of BTD which is a combination of CP decomposition~\cite{carroll1970analysis} and Tucker decomposition~\cite{tucker1966some}. The difference is that three factor matrices $Q,~K$ and $V$ are shared in constructing each $3$-order block tensor. This process can lead to reduce many parameters. After that, the average of each block tensor can be computed. In this process, each block tensor is named as Single-block attention, and the model is named Multi-linear attention (i.e., Multi-linear attention). 
% We also prove that the $3$-order tensor can reconstruct the scaled dot-product attention in Transformer by a sum on a particular dimension. 
% We apply the operators of tensor split and matrices concatenate for Multi-linear attention. Therefore, it is easily integrated into the encoder and decoder framework of Transformer~\cite{vaswani2017attention,dai2019transformer} and trained end-to-end.

Our method combines two ideas which are the low-rank approximation and parameters sharing at the same time. Therefore, it achieves the higher compression ratios. Although the self-attention (i.e., scaled dot-product attention) in Transformer can be reconstructed, we do not consider reconstructing it and choose to split the $3$-order tensor (the output of Multi-linear attention) which is helpful for improving the accuracy in experiments. 
% The representation of Multi-linear attention contains more complex and detailed attention information than the scaled dot-product attention or multi-head attention.
      
Our major contributions of this paper are as follows:
\begin{itemize}\setlength{\itemsep}{0pt}
  \item[1)] It is proved that the output of scaled dot-product attention (considering as a function) can be linearly represented by a group of orthonormal base vectors.  
  \item[2)] A novel self-attention method, namely Multi-linear attention, is provided, which combines two compression ideas, parameters sharing and low-rank approximation, together.
  \item[3)] Multi-linear attention builds the strong connection between three factor matrices (pack a set of queries, keys and values, respectively ), enhancing the ability of capturing sufficient attention information. We also prove our model can reconstruct the scaled dot-product attention in the original Transformer.
\end{itemize}

In order to validate the benefits of our model, we test it on two NLP tasks, namely language modeling and neural machine translation. In our experiments, the multi-head attention can be replaced by the proposed model, namely multi-linear attention.
% multi-linear attention can compress the multi-head attention layer at 
We have observed that the standard Multi-head attention can be compressed with higher compression ratios on One-Billion dataset. As a result, we show that multi-linear attention not only considerably reduces the number of parameters, but also achieve promising experiments results, especially in language modeling tasks.
% Our method about compressing in Transformer is differ from the tensor-train~\cite{oseledets2011tensor} methods. 
% There is the limitation of fixed context length in Transformer, which make that the learning ability of the previous models is limited. Transformer-XL can go beyond a fixed length to learn dependency.

\section{Preliminaries} 
\label{previous:work}
Multi-linear attention is carried out in this paper. The analysis of Multi-linear attention relies on these concepts and results from the field of tensor decomositon and multi-head attention. 
% The analysis of Multi-linear attention is carried out in this paper rely on concepts and results from the filed of tensor decomposition and multi-head attention.
We cover below in Section~\ref{BTD-decompositon} basic background on Block-Term tensor decomposition~\cite{de2008decompositions}. Then, we describe in Section~\ref{multi-head attention} multi-head attention~\cite{vaswani2017attention}.
\subsection{Tensor and Block-Term Tensor Decomposition}
\label{BTD-decompositon}
\textbf{Tensor} We use the Euler script letter $\mathcal{A}$ to denote a tensor which can be thought of as a multi-array. Thereby a vector and a matrix are a $1$-order tensor and $2$-order tensor, respectively.
The element in a $n$-order tensor is denoted as $\mathcal{A}_{d_1,\ldots,d_n}$. In the geometric representation of a tensor, $3$-order tensor can be represented by a cube. After that, there is a related concept named $tensor~slice$ that will be used in this paper. Tensor and some other related concepts are showed in Supplementary Materials A.

\noindent\textbf{Block-Term Tensor Decomposition (BTD)} Block-Term tensor decomposition is a combination of CP decomposition~\cite{carroll1970analysis} and Tucker decomposition~\cite{tucker1966some}. Given a $n$-order tensor $\mathcal{A} \in \mathbb{R}^{d_1\times \ldots \times d_n}$. A high-order tensor can be decomposed into $P$ block terms by the method named BTD. ${\bullet}_z$ is denoted as the tenor-tensor product on the $z$-$th$ order~\cite{kolda2009tensor} and $z\in\{1,\ldots,d\}$. Each term contains ${\bullet}_z$ between a core tensor $\mathcal{G}_i \in \mathbb{R}^{R_1 \times \ldots \times R_d}$ and $d$ factor matrices $\mathcal{X}_i^{(k)} \in \mathbb{R}^{d_k \times R_k}$, where $i \in [1,P]$ and $k \in [1, d]$. The formulation of BTD decomposition is as follows:
\begin{equation}
\label{BT-eq}
\mathcal{A} = \sum_{i=1}^{P} \mathcal{G}_i {\bullet}_1 \mathcal{X}_i^{(1)} {\bullet}_2 \mathcal{X}_i^{2} {\bullet}_3 \ldots {\bullet}_d \mathcal{X}_i^{(d)}
\end{equation} 
where $P$ is the CP rank, and $d$ is the Core-order. 
In our work, the tensor is $3$-order. Figure~\ref{BT_3-order} demonstrates the example of how a $3$-order tensor $\mathcal{A}$ can be decomposed into $P$ block terms.
\begin{figure}[t]
  \centering
  \includegraphics[width=3.5in]{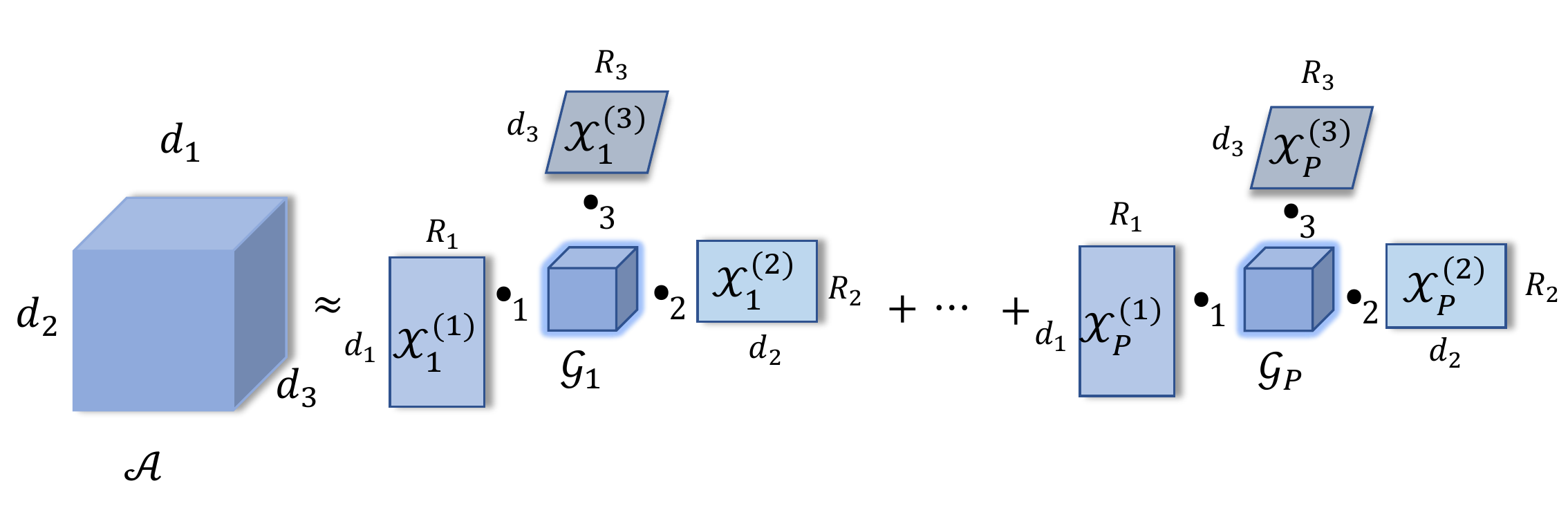} 
  \caption{The representation of Block-Term tensor decomposition for a $3$-order tensor. $\mathcal{A} \in \mathbb{R}^{d_1 \times d_2 \times d_3}$ is a $3$-order tensor, and can be approximated by $P$ Tucker decomposition. $P$ is the CP rank, and $R_1, R_2, R_3$ are the Tucker rank, respectively. In this paper, we assume that $R$=$R_1$=$R_2$=$R_3$.}
  \label{BT_3-order} 
\end{figure}

\subsection{Multi-head Attention}
\label{multi-head attention}
In Transformer, the attention function is named as ``Scaled Dot-Product Attention''. In practice, Transformer~\cite{vaswani2017attention} processes \textit{query, keys} and \textit{values} as matrices $Q$, $K$, and $V$ respectively. The attention function can be written as follows:
\begin{equation}
\label{attention}
Attention(Q,K,V) = softmax(\frac{QK^{T}}{\sqrt{d}})V
\end{equation}
where $d$ is the number of columns of $Q$ and $K$. In these work~\cite{vaswani2017attention,devlin2018bert,dai2019transformer}, they all use the multi-head attention, as introduced in~\cite{vaswani2017attention},
\begin{equation}
\begin{aligned}
MultiHeadAttention(Q,K,V) &=
 Concat(head_1,\ldots,head_k)W^{O}\\
where~head_i &= Attention(QW^{Q}_{i}, KW^{K}_{i},VW^{V}_{i})
\end{aligned}
\label{muti-head-attention}
\end{equation}
where matrices $W^{Q}_{i}$ and $W^{K}_{i}\in\mathbb{R}^{d_{model}\times{d}}$, $W_{i}^{V} \in \mathbb{R}^{d_{model}\times{d}}$ and $W^O \in \mathbb{R}^{hd_v\times d_{model}}$. In practice, $d_v$ is equal to $d$.
In this work~\cite{vaswani2017attention}, multiple groups of parameters ($W_i^{Q}$, $W_i^{K}$ and $W_i^{V}$) are used, which results in a large number of redundant parameters. 
% Firstly, computing many times linear mapping for the input data and getting multiple group $Q$, $K$, and $V$. Secondly, using the Scale Dot-Product Attention in each group $Q$, $K$, and $V$. After that, combining the results of each group for linear mapping. 

\section{Tensorized Transformer}
In this section, we first build a Single-block attention in Figure~\ref{Model} (left) based on the Tucker decomposition, a low-rank decomposition method. In this process, we prove that the self-attention function in Transformer can be represented by a linear function, i.e., a linear combination representation of a set of basic vectors. 

%We also analyze Single-block attention architecture from the effective representation and compression.
% as well as relatively lower complexity.

%The multi-head attention in Transformer makes linear mapping many times. The parameters $W$ of linear mapping is different in each time. 

In order to compress the multi-head mechanism, we propose a multi-linear attention constructed by a Block-Term tensor decomposition. This attention uses the idea of parameters sharing, i.e., sharing factor matrices across multiple blocks, shown in Figure~\ref{Model} (right). After that, the compression ratios and relatively lower complexity have been analyzed.

% The multi-head attention in Transformer makes linear mapping to $Q$, $K$ and $V$ matrices many times. The parameters $W$ of linear mapping is different in each time. In order to compress and build a multi-head mechanism, the method of BTD can be used. Different from the common BTD, the factor matrices are shared when construct each block tensor. The factor matrices are from the outputs of the first time linear mapping for $Q$, $K$ and $V$ respectively, 
% The multi-linear attention is built. Figure~\ref{Model} (right) shows the basic architecture of multi-linear attention. 
% Here, we first briefly show how to reconstruct the self-attention from Single-block attention to multi-linear attention (i.e., Multi-linear attention). 

\begin{figure}[t]
\centering
\includegraphics[width=5.0in]{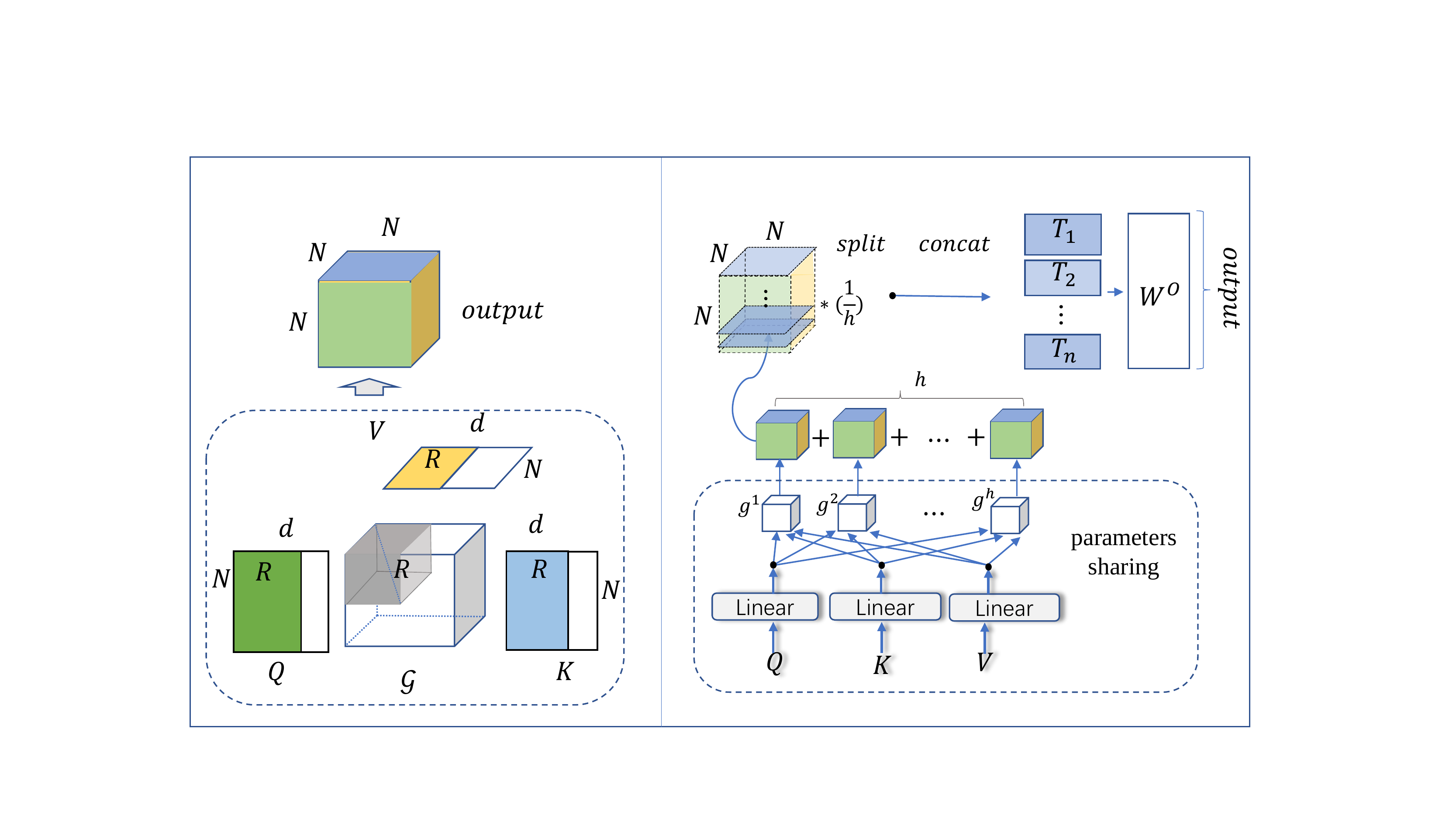}
\caption{(left) Single-block attention using Tucker decomposition. (right) Multi-linear attention based on Block-Term tensor decomposition.}
\label{Model}
% \vspace{-10px}
\end{figure}
\subsection{Single-block Attention by Tucker Decomposition}
Before building the Single-block attention, it is necessary to propose the theorem~\ref{theorem1}. The theorem is closely related to attributes of Single-block attention function by Tucker decomposition~\cite{tucker1966some}. 

\begin{theorem}
\label{theorem1}
Let $\bm{e}_1, \ldots, \bm{e}_n$ be basis vectors from the vector space $S$. Assume that these vectors $\bm{e}_1,\ldots,\bm{e}_n$ are linear independent and $Q$,$K$,$V$ can be linearly represented by this set of basis vectors. The output of the attention function in Eq.~\ref{attention} can be represented by a linear combination of the set of these basis vectors.
\begin{equation}
\label{eq4}
Attention(Q,K,V) = (\bm{e}_1, \ldots, \bm{e}_n)M,
\end{equation}
where $M \in \mathbb{R}^{n\times d}$ is a coefficient matrix, and $d$ is a dimension of these matrices (i.e., $Q,~K$, and $V$).
\end{theorem}
\begin{proof}
% \vspace{-10px}
The proof can be found in Supplementary Materials B. 
\end{proof}

% In Corollary~\ref{corollary}, we prove that the Single-block attention can reconstruct the self-attention in Transformer.
% Based on Theorem ~\ref{theorem1} and Corollary~\ref{corollary}, it turns out that $M$ in Eq.~\ref{eq4} can be computed by Eq.~\ref{dddd}

In Figure~\ref{Model} (left), it is a schematic diagram about the Single-block attention. First, we assume that the query, key and value can be mapped into three factor matrices of which are composed of three groups of orthogonal basis vectors. Three factor matrices are $Q$, $K$ and $V$. After that, we can construct a new attention (i.e., Single-block attention) by initializing a $3$-order diagonal tensor (trainable) which is the $\mathcal{G}$. In Figure~\ref{Model} (left), $R$ is the rank about the tensor, $N$ is the length of a sequence, and $d$ is the dimension of matrix. The function of Single-block attention can be computed based on Tucker-decomposition as follows:
\begin{equation}
\label{New_Attention}
\begin{aligned}
Atten_{TD}(\mathcal{G};Q,K,V) =& \mathcal{G} {\bullet}_1 Q {\bullet}_2 K {\bullet}_3 V \\
=& \sum_{i=1}^{I}\sum_{j=1}^{J} \sum_{m=1}^{M} \mathcal{G}_{ijm} Q_i \circ K_j \circ V_m 
\end{aligned}
\end{equation} 
where $\mathcal{G}$ is a core tensor. $i, j$ and $m$ are the indexes of the core tensor. $\circ$ is the outer product. ${\bullet}_z$ is the same definition in Eq.~\ref{BT-eq}. $Q_i, K_j$ and $V_k$ are column vectors from matrices $Q, K$ and $V$, where $Q \in \mathbb{R}^{N \times d}$, $K \in \mathbb{R}^{N \times d}$ and $V \in \mathbb{R}^{N \times d}$, and $N$ is the length of a sequence. In practice, we set $I$=$J$=$M$=$R$. The core tensor $\mathcal{G}$ can be defined as follows,
\begin{equation}
\label{core_tensor}
\mathcal{G}_{ijm} =
 \left\{  
        \begin{array}{lr}
          rand(0,1) & i=j=m \\
          0 & otherwise\\
        \end{array} 
\right.
\end{equation} 
where the $rand(0,1)$ is a random function, and the diagonal entries of core tensor $\mathcal{G}$ form the vector $\bm{g}$. Each entry $\bm{g}_r\in (0,1)$, $r \in \{1, \ldots, R\}$. We can consider $\bm{g}$ as the trainable weight. In experiments, we compute the weight vector by $softmax$ function (i.e., $softmax(\bm{g})$).

%In our design, the basic idea is that a $3$-order tensor is used as the outputs of the attention functions based on the form of Tucker decomposition~\cite{tucker1966some}.

After that, the output of Single-block attention function is a $3$-order tensor which is given by linear computation. The Single-block attention (i.e., a $3$-order tensor with Tucker decomposition) can reconstruct the {\it{Scaled Dot-Product attention}} in Eq.~\ref{attention} by the summing over the tensor according to the second index~\footnote{If the coordinates of a $3$-order tensor are $i,~j$ and $m$, $j$ is the second index.} (it can be seen as the coordinates in the vertical direction for a tensor), as proved in the following corollary. Note that in our model, we do not adopt the above reconstructing process. Instead, to obtain a new representation, we adopt the concat method after the tensor splitting (see Sec.~\ref{sec:sec:multiBT}). We will further show the compression ability of the Single-block attention in Sec.~\ref{sec:sec:compressanalysis}.

%The following corollary explain the representation ability of the Single-block attention. 

% Therefore, Multi-linear attention can model more complex output representation. After that, the {\it{corollary}} is defined as follows.

\begin{corollary}
\label{corollary}
Under the same conditions as in Theorem~\ref{theorem1} and the value of $N$ is equal to the value of $d$, Single-block attention representation Eq.~\ref{New_Attention} can reconstruct the {\it{Scaled Dot-Product attention}} in Eq.~\ref{attention} by the summing over the tensor (i.e., the output of Single-block attention function) according to the second index. It holds that:
\begin{equation}
\label{dddd}
Attention(Q,K,V)_{i,m} = \sum_{j=1}^{N} Atten_{TD}(\mathcal{G};Q,K,V)_{i,j,m}
\end{equation}
where $i$, $j$ and $m$ are the indices of the Single-block attention's output (i.e., a 3-order tensor). $Atten_{TD}(\cdot)$ is the function of Single-block attention based on Tucker decomposition. $i$ and $m$ are the indices of outputs (i.e., a matrix) from Eq.~\ref{attention}.
\end{corollary}
\begin{proof}
% \vspace{-12px}
The proof can be found in Supplementary Materials C.
\end{proof}

\subsection{Multi-Linear Attention by Block-Term Tensor Decomposition}
\label{sec:sec:multiBT}
In order to construct the multi-head mechanism and compress the parameters of multiple groups of mapping, we use a group of linear projections, and share the output from the linear projections. In Figure~\ref{Model}(right), the learned linear projection can map queries, keys and values to three matrices which are composed of basis vectors. After that, we use the Block-Term tensor decomposition to build multi-head mechanism. In our work, our model is named as Multi-linear attention, which can be formulated as follows:
\begin{equation}
\begin{aligned}
MultiLinear(\mathcal{G};Q,K,V) &= SplitConcat(\frac{1}{h}*({T}_1+ \ldots +{T}_h))W^{O} \\
where~~{T_j} &= Atten_{TD}(\mathcal{G}_j;QW^{q},KW^{k},VW^{v})
\end{aligned}
\label{Multi_linear_attention}
\end{equation}
where the core tensor $\mathcal{G}_j$ is a diagonal tensor, and the number of parameter in $\mathcal{G}_j$ is equal to the rank of core tensor, $j\in \{1,\ldots, h\}$. $\mathcal{G}$ is the set of the core tensors. $SplitConcat(\cdot)$ is a function which achieves the concatenation after splitting for a $3$-order tensor. Figure~\ref{Model} (right) shows the basis idea about the multi-linear attention. The $W^{O}$ is the parameter matrix which is a full connection layer and correlated to the output of Multi-linear attention. $Atten_{TD}(\cdot)$ is the function of Single-block attention, which is a part of Multi-linear attention. $W^{q}$, $W^{K}$ and $W^{v}$ are the parameters matrices which are shared in constructing Multi-linear attention.

The Multi-linear attention is a compression model. After compressing the multi-head attention in Transformer, it is to achieve a Tensorized Transformer. The Multi-linear attention can be incorporated into Transformer architecture. A diagram which is about the incorporating of Multi-linear attention in partial Transformer structure is given in Supplementary Materials E.1.

\subsection{Analysis of Compression and Complexity}
\label{sec:sec:compressanalysis}
\textbf{Compression} 
Our focus is on the compression of the multi-head mechanism in the multi-head attention of Transformer. Previous work~\cite{vaswani2017attention} gets the multi-head attention by multiple groups of linear mappings. We use three linear mappings for matrices $Q$,~$K$ and $V$, respectively. For the output of three mappings, we choose to share them which are considered as three factor matrices in reconstructing the Multi-linear attention. This process is shown in Figure~\ref{Model} (left). $h$ is the number of heads in~\cite{vaswani2017attention}, and $d$ is the dimension of factor matrices. The compression ratios can be computed by 
$({3 \times h \times d})/({3 \times d + h})$. In practice, $h$ is normally set to $8$, $d$ is set to $512$. In this case, the compression ratios can achieve $8$. In other words, we can reduce almost $8$ times parameters in the attention layer.
The details of the computing of compression ratios can be found in Supplementary Materials D.
The Transformer also contains other network layers, such as Position-wise feed forward network and embedding layers et al. Therefore,
for the compression ratios in whole Transformer, we can compare it by the analysis of experimental results for model parameters.

\textbf{Complexity} The time complexity of the attention function in Eq.~\ref{attention} is $\mathcal{O}(N^2d)$, $N$ is the length of a sequence, and $d$ is the representation dimension. In Multi-linear attention, we can reorder the computations to receive the model complexity $\mathcal{O}(N^3)$, where $N$ is also the length of the sequence. The minimum number of sequential operations in Multi-linear attention for different layers is approximately equal to the self-attention in Transformer~\cite{vaswani2017attention}. 
% In our experiments, $N$ is smaller than $d$. Therefore, our model consumes less time than the original transformer.

\section{Related Work}
%In recent years, deep neural networks have received lots of attention, been applied to different applications and achieved dramatic accuracy improvements in more and more tasks. These works rely on deep networks with millions or even billion of parameters. In practice, neural networks tend to be dramatically over parameterized. 

The field of language modeling has witnessed many significant advances. Different from the architectures of convolutional neural network (CNNs) and recurrent neural networks (RNNs) language modeling, the Transformer~\cite{vaswani2017attention} and its variants~\cite{dai2019transformer,devlin2018bert,dehghani2018universal} achieve excellent results in language modeling processing. Transformer networks have a potential of learning long-term dependency, but are limited by a fixed-length context in the setting of language modeling. Vaswani \textit{et al.}~\cite{vaswani2017attention} uses a segment-level recurrence mechanism and a novel positional encoding scheme to resolve this question. BERT~\cite{devlin2018bert} is a kind of bidirectional encoder representations from transformers. It is designed to pre-train deep bidirectional representation and obtains new SoTA on some NLP tasks. Although these methods have achieved great results, a large number of parameters make it difficult for the model to be trained in limited resources. Transformer fails to generalize in many simple tasks, e.g. copying string and logical inference~\cite{dehghani2018universal}. Universal Transformers~\cite{dehghani2018universal} propose a self-attentive recurrent sequence model which addresses this problem. This methods can increase the training speed. In their work, authors following weight sharing found in CNNs and RNNs, extend the Transformer with a simple form of weight sharing that strikes an effective balance between induces and model expressivity. This methods also uses a large number of parameters. 

Therefore, it is very important to consider how to reduce the amount of memory and computing they need. As we know, existing model compression methods are mainly divided into parameter pruning and sharing~\cite{han2015learning}, low rank approximation~\cite{sainath2013low}, knowledge transfer~\cite{buci2006model}, and transferred convolutional filters~\cite{cohen2016group}.  
Currently, tensor decomposition methods are used to decompose a high-order tensor, which can get different neural network language model structures~\cite{zhang2019generalized,zhang2018quantum}. Besides, tensor decomposition methods which adopts the idea of low rank approximation in most cases, have been successfully applied to neural networks compression. For example, in literature~\cite{denton2014exploiting,jaderberg2014speeding}, researchers approximate a tensor by minimizing the reconstruction error of the original parameters on convolutional neural networks (CNNs). However, these approaches tend to accumulate errors when multiple layers are compressed sequentially, and the output feature maps deviate far from the original values with the increase of compressed layers. Our compression method uses the idea of parameters sharing in the constructing of attention layers, and the size of output is same as the output from self-attention in Transformer which can effectively avoid these problems.
Tensorizing Neural Networks~\cite{novikov2015tensorizing} have combined the idea of reshaping weights of fully-connected layers into high-dimensional tensors and representing them in Tensor Train format~\cite{oseledets2011tensor}. This approach was later extended to convolutional~\cite{garipov2016ultimate} and recurrent neural networks~\cite{yang2017tensor}.
Recently, in these work~\cite{chen2018groupreduce,variani2018west}, researchers introduce efficient compression methods for the embedding and $softmax$ layers based on structured low rank matrix approximation. TT-embedding~\cite{khrulkov2019tensorized} aims to compression the larger embedding layer on Transformer-XL~\cite{dai2019transformer}. Sparse Transformer~\cite{Rewon2019sparse} adopts sparse techniques on the attention matrix and reduces its parameters. This work uses a sparse attention matrix by selecting the information on some positions in the attention matrix, but does not change the mechanism of the attention.
Our method is different from these works, and combines two compression idea (low rank approximate and parameters sharing) to construct a tensorized Transformer. 

% Although this method reduce the parameters, it causes the increasing of complexity.

In our work, we focus on the compression the multi-head attention in Transformer based the idea of parameters sharing. At the same time, we also combine low-rank approximate method to reduce parameters and computation complexity. 

\section{Experiments}
Transformer is a versatile and powerful modeling tool and widely is used in various natural language process tasks. In order to verify the effectiveness of our method (i.e., Multi-linear attention) replacing multi-head attention in Transformer, we carry out two NLP tasks named language modeling (LM) and neural machine translation (NMT). Code\footnote{https://github.com/szhangtju/The-compression-of-Transformer} for running experiments has been released, and the key code which is about our method can be found in Supplementary Materials F.

\subsection{Language Modeling}
Language modeling is the task of predicting the next word in a sentence. This task is to estimate the joint probability $p(s)$ of a sentence of tokens $s$=$(w_1,\ldots, w_n)$. The resulting models can be used to generate text or further fine-tuned to solve other NLP tasks~\cite{radford2018improving}. 
In this paper, we employ the standard setting of predicting next token given the sequence of preceding tokens, based on the function $p(s)=p(w_1)\prod_{i=2}^n p(w_i|w_1,\ldots,w_{i-1})$. We chose three datasets in the order of small (i.e., PTB), medium (i.e., WikiText-103) and large (i.e., One-Billion). Models are evaluated based on Perplexity (PPL), which is the average per-word log-probability. The lower the PPL, the better the model is.

Specially, we take Transformer, the open source state-of-the art language modeling architecture, and replace the standard multi-head attention layers with our Multi-linear attention. Then, we test different model configurations on the PTB~\cite{mikolov2011empirical}, WikiText-103~\cite{merity2016pointer} and One-Billion Word benchmark~\cite{chelba2013one} datasets and report the results in Table~\ref{Tabel1} and Table~\ref{result}. 

\begin{table*}[ht]\small
\centering 
\caption{Results (PPL) and model parameters with state-of-the-art results on One-Billion. Tensorized Transformer is our model. The core-1 is that the model use Single-block term tensor. Analogously, the core-2 is that two block term tensor is used.}
\begin{tabular}{ccccc}
\toprule[1pt]
\textbf{Model} & \textbf{Params} & \textbf{Test PPL}\\
\midrule[0.5pt]
RNN-1024+9 Gram~\cite{chelba2013one} & 20B & 51.3\\
LSTM-2018-512~\cite{jozefowicz2016exploring} & 0.83B & 43.7 \\
GCNN-14 bottleneck~\cite{dauphin2017language} &--& 31.9 \\
LSTM-8192-1024+CNN Input~\cite{jozefowicz2016exploring}  &1.04B&30.0 \\
High-Budget MoE~\cite{shazeer2017outrageously} &5B &28.0 \\
LSTM+Mos~\cite{yang2017breaking} & 113M & 37.10 \\
\hline
Transformer+adaptive input~\cite{baevski2018adaptive} & 0.46B & 23.7 \\
Transformer-XL Base~\cite{dai2019transformer} &0.46B& 23.5 \\
Transformer-XL Large~\cite{dai2019transformer} & 0.8B& 21.8 \\
\hline 
Tensorized Transformer core-$1$ &0.16B& 20.5 \\
Tensorized Transformer core-$2$ &0.16B& \textbf{19.5} \\
\bottomrule[1pt]
\end{tabular}
\label{Tabel1}
% \vspace{-11px}
\end{table*}

\begin{table*}[pthb]\small
\renewcommand\arraystretch{1.2}
  \centering
  \begin{tabular}{llllllll}
    \toprule[1pt]
    \multirow{2}{*}{\textbf{Model}}& \multicolumn{3}{c}{\textbf{PTB}} & \multicolumn{3}{c}{\textbf{WikiText-103}}\\
    \cline{2-7}
    &Params&Val PPL&Test PPL&Params&Val PPL&Test PPL\\
    \hline
    LSTM+augmented loss~\cite{inan2016tying}    & 24M & 75.7 & 48.7   & -- & -- & 48.7  \\
    Variational RHN~\cite{zoph2016neural} & 23M & 67.9 & 65.4  & --& -- & 45.2 \\
    % GCNN-14 ~\cite{dauphin2017language} & -- &--& --  &-- &--& 37.2 \\
    4-layer QRNN~\cite{merity2018analysis} & -- &--& --  &151M &--& 33.0\\
    AWD-LSTM-MoS~\cite{yang2017breaking} & 22M &58.08& 55.97  & -- &29.0& 29.2 \\
    \hline
    Transformer+adaptive input~\cite{baevski2018adaptive} & 24M & 59.1 & 57     & 247M & 19.8 & 20.5\\
    Transformer-XL-Base~\cite{dai2019transformer} & 24M &56.72& 54.52               &151M &23.1& 24.0\\
    Transformer-XL-Large~\cite{dai2019transformer} &--& -- &--                      &257M &--& \textbf{18.3} \\
    Transformer-XL+TT~\cite{khrulkov2019tensorized}& 18 M & 57.9* & 55.4*           &130M &23.61* & 25.70* \\
    Sparse Transformer~\cite{Rewon2019sparse} & 14M &74.0*&73.1*                 & 174M & 38.98* &40.23* \\
    \hline
    Tensorized Transformer core-$1$ & 12M& 60.5 & 57.9 &  85.3M &22.7&20.9\\
    Tensorized Transformer core-$2$ & 12M& \textbf{54.25} & \textbf{49.8} &  85.3M &\textbf{19.7}&18.9\\
    % NNQLM-\uppercase\expandafter{\romannumeral2} & 0.7496& 0.8124 & 0.6496& 0.6594&0.591&0.755&0.691&0.812&1\\
    \bottomrule[1pt]
  \end{tabular}
  \caption{Results and compression with state-of-the-art results on PTB and WikiText-103. '$-$' indicates no reported results in that setting, '$*$' indicates that the results is our own implementation. }
  \label{result}
  % \vspace{-11px}
\end{table*}

%Our-Tensorized Transformer is our model. The core-1 is that the model use Single-block term tensor. Analogously,  the core-2 is that two block term tensor is used.

\subsection{Results and Details}
% We chose three datasets in the order of small (i.e., PTB), medium (i.e., WikiText-103) and large (i.e., One-Billion). As part of pre-processing, words are completely lower case. Newlines were replaced with <eos>. The vocabulary is the most frequent words with the rest of the tokens replaced by an <UNK> token. Models are evaluated based on Perplexity (PPL), which is the average per-word log-probability. the lower the PPL, the better the model is.
% We report the results on the word-level Penn TreeBank in \textbf{Table}. Similar to AWD-LSTM (\textbf{cite}), we apply variational dropout and weight average to Transformer to Transformer-BTD. {\textit{Our model achieved the comparable results}} than Transformer or Transformer-XL model.  The results of Our model is better than other LSTM model.
% With proper regularization, the Transformer-BTD achieves a new SoTA result among models with two-step finetuning.
PTB has $929k$ training tokens, $73k$ validation words, and $82k$ test words. The results is reported in Table~\ref{result}. Similar to AWD-LSTM-MoS~\cite{yang2017breaking}, we apply variational dropout and weight average to our model (i.e., Tensorized Transformer). In addition, we need to state that, our model only replaces the multi-head attention using Multi-linear attention structure, and the other structures remain the same. We compare the results of our model with other models. Our model achieves the comparable results with SoTA when the number of core tensor is equal to two. However, our model size (i.e, model parameters) reduces by nearly half comparing with Transformer and Transformer-XL. 
% After that, we need to state that our method will be over-fitting when the core is greater than $2$.

WikiText-103 contains 267,735 unique tokens. The dataset is available word-level language modeling benchmark with long-term dependency. It contains 103M training tokens from $28k$ articles, with an average length of 3.6k tokens per article, which allows testing the ability of long-term dependency modeling. 
% Here, we set the sentence length is $100$, which is different from the sentence length in PTB ($30$) and One-Billion ($30$). 
As shown in Table~\ref{result}, our model get the perplexity of $18.9$, which is a comparable experimental result with the previous SoTA perplexity $18.3$ , which demonstrates the effectiveness of the proposed attention architecture. 

The One-Billion Word benchmark is a large dataset derived from a news site. The dataset consists of $829,250,940$ tokens over a vocabulary of $793,471$ words. In this dataset, sentences are shuffled and hence the context is limited. Consequently, this dataset mainly tests the ability of modeling only short-term dependency. The comparison between Tensorized Transformer and the other methods are shown in Table~\ref{Tabel1}. Although Tensorized Transformer is mainly designed to better compress Transformer or Transformer-XL model, it dramatically improves the single-model SoTA from \textbf{21.8} to \textbf{19.5}. Specifically, Tensorized Transformer significantly outperforms a contemporary method using vanilla Transformers~\cite{vaswani2017attention}, suggesting that the advantage of the tensorized Transformer is also generalizable to modeling short sequences.

Table~\ref{result} and Table~\ref{Tabel1} show that our model get the lower PPL than other models in three datasets. An exciting observation is that our model has much fewer parameters.
The model of Transformer-XL+TT~\cite{khrulkov2019tensorized} is a recent compression model with Tensor Train to compress the input embedding layers only. Sparse Transformer~\cite{Rewon2019sparse} uses the method of sparse attention matrix to compress Transformer model. The results in Table~\ref{result} show that compared with Transformer-XL+TT, our method has much fewer parameters, and better language modeling performance. 
These results verify that our model (i.e., Multi-linear attention) is effective in language modeling tasks, and has performed well for the model compression.
Other details (such as hyperparameters and Hardware) can be found in Supplementary Materials E.

%that this method have more parameters than ours, but PPL increased in test and validation datasets than Transformer-XL~\cite{dai2019transformer}.
\subsection{Neural Machine Translation}
The goal is to map an input sequence $s=(x_1,x_2,\ldots,x_n)$ representing a phrase in one language, to an output sequence $y=(y_1,y_2,\ldots, y_m)$ representing the same phrase in a different language. In this task, we have trained the Transformer model~\cite{vaswani2017attention} on WMT 2016 English-German dataset~\cite{sennrich2016edinburgh}. Sentences were tokenized using the SentencePiece~\footnote{https://github.com/google/sentencepiece}. For our experiments, we have replaced each of the attention layers with Multi-linear attention in Encoder. For evaluation we used beam search with a beam size of 5 and length penalty $\alpha$=$0.6$. In this section, we only compared the results with Transformer~\cite{vaswani2017attention}. Our results are summarized in Table~\ref{BLEU}. $*$ indicates that the result is our own implementation. 

In Table~\ref{BLEU}, we select two baseline models. The Base-line~\cite{sennrich2016edinburgh} is first model in WMT 2016 English-German dataset.
For the other baseline, we use the basic Transformer architecture~\cite{vaswani2017attention}. The BLEU score is $34.5$ for the basic architecture.
We carry out two Tensorized Transformer structures, namely core-1 and core-2 respectively. When Tensorized Transformer core-1 and core-2 are used, the BLEU scores are $34.10$ and $34.91$, which achieves better performance over Transformer. As for the reported model parameter size, our model uses less parameters. 
%Table~\ref{BLEU} shows that our model can achieved the compression about Tansformer. 
% and this code~\footnotetext[2]{https://github.com/tensorflow/tensor2tensor} is followed. 
% More details and discussion about the NMT is set in Supplementary Materials D.

\begin{table}\small
\centering 
\caption{Results and compression with Transformer on WMT-16 English-to-German translation.}
\begin{tabular}{ccccc}
\toprule[1pt]
\textbf{Model} & \textbf{Params} & \textbf{BLEU}\\
\midrule[0.5pt]
Base-line~\cite{sennrich2016edinburgh} &--&26.8\\
Linguistic Input Featurec~\cite{sennrich2016linguistic} &--& 28.4 \\

Attentional encoder-decoder + BPE~\cite{sennrich2016edinburgh} &--&34.2  \\
Transformer~\cite{vaswani2017attention} &52M& 34.5* \\
\hline 
Tensorized Transformer core-$1$ &21M& 34.10 \\
Tensorized Transformer core-$2$ &21.2M& \textbf{34.91} \\
\bottomrule[1pt]
\end{tabular}
\label{BLEU}
% \vspace{-12px}
\end{table}

\subsection{Discussion}

We have shown the results on language modeling and neural machine translation tasks using the Multi-linear attention. For the compression of the model parameters, although we report the parameters of the whole model structure, our method mainly considers the compression of multi-head attention but has not changed other layers in Transformer. Regarding the rationale for the improvements, in Corollary 1, we prove that the output of the original attention can be represented by summing over the 3-order tensor. In Figure 2, we use a concat function over these matrices from tensor splitting. The operation of concat can model all values in the 3-order tensor, and thus captures more information than sum operator. Another reason could be the alleviation of overfitting by reducing parameters. The overfitting will appear when the number of the core tensor is greater than 2. Besides, according to our experiments, relatively large dimensions of the word embedding can lead to overfitting, resulting in performance degradation. Therefore, our model requires a relatively small dimension of the embedding, compared with the original Transformer. In order for a more systematic evaluation, we report more experiments and analyses in Supplementary Materials E.4.

\section{Conclusion and Further Work}
We have proposed a novel self attention encoder layer, namely the Multi-linear attention, to compress the original multi-head attention and derive a novel encoding scheme. 
Our main contribution lies in a structure of Tensorized Transformer based on Block-Term tensor decomposition which is represented by the combination of a group of $3$-order tensors, with low-rank approximation and parameters sharing ideas adopted. Compared with existing Transformer based methods, our model achieved higher compression ratio and got better experimental results, particularly in language modeling task. These evidences imply that our method can potentially be further applied to more NLP tasks with limited resources.

In the future, we will continue to optimize the Tensorized Transformer framework and apply it in other NLP tasks. As we stated earlier, our model may suffer from overfitting when the number of cores is large in language modeling. In the future, we will explore the fundamental
reasons that cause the problem and tackle them within the Tensorized Transformer framework.
 
\section{Acknowledgement}
This work is supported in part by the state key development program of China (grant No. 2017YFE0111900, 2018YFC0831704), Natural Science Foundation of China (grant No. 61772363, U1636203), and the European Unions Horizon 2020 research and innovation programme under the Marie SkodowskaCurie grant agreement No.721321.

\medskip
\small
\bibliographystyle{plain}

\appendix
\section{Tensor and Tensor Slice}
As introduced in~\cite{cichocki2009nonnegative}, a tensor and the tensor slice can be defined as follows.

\begin{definition}[tensor]
Let $D_1$, $D_2$, $\ldots$, $D_N$$\in \mathbb{N}$ denote index upper bounds. A tensor $\mathcal{A} \in \mathbb{R}^{D_1,\ldots,D_N}$ of order $N$ is an $N$-way array where elements $\mathcal{A}_{d_1,d_2,\ldots,d_n}$ are indexed by $d_n \in \{ 1,2, \ldots, D_n\} $ for $1 \le n \le N$. 
\end{definition}

The concept of tensor slice is specified as:
\begin{definition}[tensor slice]
 A tensor slice is a two-dimensional section (fragment) of a tensor, obtained by fixing all indexes except for two indexes.
\end{definition} 
\section{Theorem 3.1}
\label{theorem31}
Let $\bm{e}_1, \ldots, \bm{e}_n$ be basis vectors from the vector space $S$. Assume that these vectors $\bm{e}_1,\ldots,\bm{e}_n$ are linear independent and $Q$,$K$,$V$ can be linearly represented by this set of basis vectors. The output of self-attention function in Eq.~2 (in the paper) can be represented by a linear combination of the set of these basis vectors.

\begin{equation}
\label{eq41}
Attention(Q,K,V) = (\bm{e}_1, \ldots, \bm{e}_n)M,
\end{equation}
where $M \in \mathbb{R}^{n\times d}$ is a coefficient matrix, and $d$ is a dimension of these matrices (i.e., $Q,~K$, and $V$).

\begin{proof}
If $Q$, $K$ and $V$ $\in$ Span($\bm{e}_1, \ldots, \bm{e}_n$), the linear combination representation of matrices $Q$,$K$ and $V$ can be written as follows:
\begin{equation}
\left\{
  \begin{array}{lr}
  Q =  \left(\bm{e}_1,\bm{e}_2,\ldots,\bm{e}_n\right) (\bm{\alpha}_1,\bm{\alpha}_2,\ldots,\bm{\alpha}_d) \\
  K =  \left(\bm{e}_1,\bm{e}_2,\ldots,\bm{e}_n\right)(\bm{\beta}_1,\bm{\beta}_2,\ldots,\bm{\beta}_d) \\
  V =  \left(\bm{e}_1,\bm{e}_2,\ldots,\bm{e}_n\right)(\bm{\xi}_1,\bm{\xi}_2,\ldots,\bm{\xi}_d)
  \end{array}
\right.
\end{equation}
The self-attention function is written as follows~\cite{vaswani2017attention}:
\begin{equation}
Attention(Q,K,V)=softmax(\frac{QK^T}{\sqrt{d}})V,
\end{equation}
where $QK^T$ can be computed as follows:
\begin{equation}
\begin{aligned}
QK^T = \left(\bm{e}_1,\bm{e}_2,\ldots,\bm{e}_n\right) (\bm{\alpha}_1,\bm{\alpha}_2,\ldots,\bm{\alpha}_d)(\bm{\beta}_1,\bm{\beta}_2,\ldots,\bm{\beta}_d)^{T}\left(\bm{e}_1,\bm{e}_2,\ldots,\bm{e}_n\right)^T
\end{aligned}
\label{td}
\end{equation}
% Accoring to Eq.~\ref{td}, it is the coefficient matrix product that represents $QK^T$. 
As a result, the input of $softmax$ function is a product of coefficient matrices $(\bm{\alpha}_1,\ldots, \bm{\alpha}_d)$ and $(\bm{\beta}_1, \ldots,\bm{\beta}_d)^T$.
Then, we have 
\begin{equation}
\label{linear}
softmax(\frac{QK^T}{\sqrt{d}})=(\bm{e}_1,\ldots,\bm{e}_n)softmax(A/\sqrt{d})(\bm{e}_1,\ldots,\bm{e}_n)^T
\end{equation}
where the matrix $A$ is equal to $(\bm{\alpha}_1,\ldots,\bm{\alpha}_d)(\bm{\beta}_1,\ldots,\bm{\beta}_d)^T$. Therefore, the attention representation can be written as follows:
\begin{equation}
\begin{aligned}
softmax(\frac{QK^T}{\sqrt{d}})V &= \left(\bm{e}_1,\bm{e}_2,\ldots,\bm{e}_n\right)softmax(A/\sqrt{d})(\bm{\xi}_1,\bm{\xi}_2,\ldots,\bm{\xi}_d)\\
& = \left(\bm{e}_1,\bm{e}_2,\ldots,\bm{e}_n\right)M
\end{aligned}
\end{equation}
where the matrix $M$ is equal to $softmax(A/\sqrt{d})(\bm{\xi}_1,\bm{\xi}_2,\ldots,\bm{\xi}_d)$.
The $softmax(A/\sqrt{d})$ is to normalize the coefficient matrices of $Q$ and $K$.
It turns out that the output of the attention function~\cite{vaswani2017attention} can be represented by a linear combination of the set of basic vectors. 
\end{proof}

After the proof, it is helpful to describe the basic idea. 
First, we consider that the self-attention function can be linearly represented by a set of orthogonal basis vectors, when the input of $softmax$ function is the product of two coefficient matrices, $(\bm{\alpha}_1,\bm{\alpha}_2,\ldots,\bm{\alpha}_d)$ and $(\bm{\beta}_1,\bm{\beta}_2,\ldots,\bm{\beta}_d)^T$, respectively.
Second, in constructing the multi-head mechanism, the matrices of basis vectors $\left(\bm{e}_1,\bm{e}_2,\ldots,\bm{e}_n\right)$ can be shared.

\section{Corollary 1}
\label{corollary1}
Under the same conditions as in Theorem 3.1 and the value of $N$ is equal to the value of $d$, the Single-block attention representation Eq.~$5$ (in the paper) can reconstruct the {\it{Scaled Dot-Product attention}} in Eq.~$2$ (in the paper) by the summing over the tensor (i.e., the output of Single-block attention function) according to the second index. It holds that:
\begin{equation}
\label{ddddd}
Attention(Q,K,V)_{i,m} = \sum_{j=1}^N Atten_{TD}(\mathcal{G};Q,K,V)_{i,j,m},
\end{equation}
where $i$, $j$ and $m$ are the indices of the Single-block attention output (i.e., a 3-order tensor). $Atten_{TD}(\cdot)$ is the function of the Single-block attention based on Tucker decomposition. $i$ and $m$ are the indices of outputs (i.e., a matrix) from Eq.~$2$ (in the paper).

\begin{proof}
In Theorem~\ref{theorem1}, we have proved the results about the attention function can be represented by a linear combination of basis vectors. Therefore, we can represent the self-attention function~in Eq.~2 (in the paper) by the form as follows:
\begin{equation}
Attention(Q,K,V) = \Theta QK^{T}V
\end{equation}
where $\Theta$ is a normalization factor matrix, which can be used to replace the use of a $sofmax$ function. We assume that $\Theta$ contains all the non-zero elements of the core tensor $\mathcal{G}$.
The self-attention in Eq.~2 (in the paper) can be re-written as follows:
\begin{equation}
\label{oldatt}
X_{i,m}=\sum_{k=1}^{N}\sum_{r=1}^{R}{\Theta}_{i,m}Q_{i,r}K_{k,r}V_{k,m}
\end{equation}
where $N$ is the length of a sentence, $X_{i,m}=Attention(Q,K,V)_{i,m}$ is the entry of the output from the self-attention, and $R$ is equal to $d$.
Here the core tensor $\mathcal{G}$ is same as that in Eq.~$7$ (in the paper).
% ~\ref{core_tensor}
Then, the Single-block attention (a $3$-order tensor) can be represented as follows:
\begin{equation}
  \mathcal{A}_{i,j,m}=\sum_{p}^{R}\sum_{q}^{R}\sum_{r}^{R}\mathcal{G}_{p,q,r}Q_{i,p}K_{j,p}V_{m,r} 
  \label{newatt}
\end{equation}
where $\mathcal{A}$ is a $3$-order tensor, 
which is equal to $Atten_{TD}(\mathcal{G};Q,K,V)$. 
Accordingly, $\mathcal{A}_{i,j,m}$ 
is a entry in tensor $\mathcal{A}$ and is equal to ${Attention_{TD}}_{i,j,m}$ in Eq.~\ref{ddddd}. 
Next, we aim to prove Eq.~\ref{dddd} can be established. Therefore, we need to establish the relation between
Eq.~\ref{newatt} and Eq.~\ref{oldatt}. Since the core tensor $\mathcal{G}$ is a special tensor (i.e., diagonal tensor), Eq.~\ref{newatt} can be written as follows:
\begin{equation}
\mathcal{A}_{i,j,m}=\sum_{r=1}^{R}\mathcal{G}_{r,r,r}Q_{i,r}K_{j,r}V_{m,r}
\end{equation}
After that, we can compute the attention representation through adding to model $k$. For better understanding, we give the graph representation in Figure~\ref{Appendix}.
\begin{figure}[t]
\centering
\includegraphics[width=4.5in]{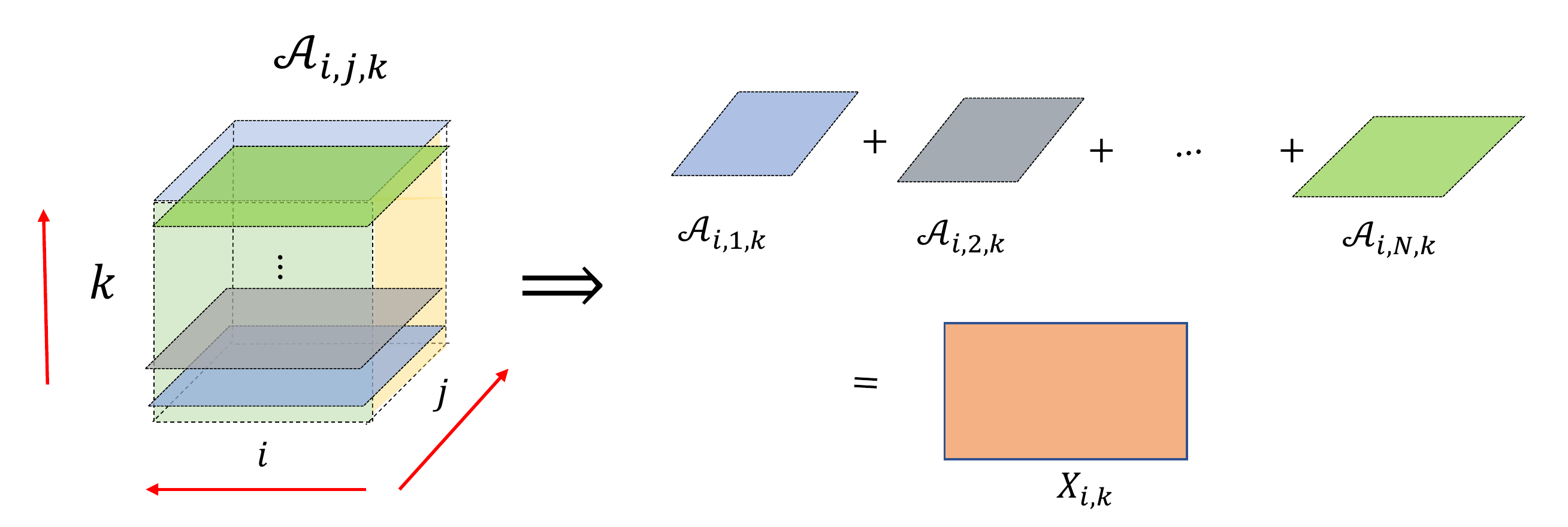}
\caption{ Tensor $\mathcal{A}$ is a $3$-order tensor, which represents the Single-block attention in the left. $\mathcal{A}_{i,j,k}$ is the entry of the tensor $\mathcal{A}$. In the right, the graph represents that 
the summing of tensor slices which is from the tensor splitting in index $j$. This graph can help us to understand the main content of corollary~\ref{corollary}.}
\label{Appendix}
\end{figure}

$$X_{i,m} = \sum_{r=1}^{R}\sum_{j=1}^{N}\mathcal{G}_{rrr}Q_{i,r}K_{j,r}V_{m,r}$$
The corollary then holds.
\end{proof}
\section{Compression Ratio about Multi-Linear Attention}
In order to compute the compression ratio, we need to compare multi-linear attention with multi-head attention. The comparison chart has been given in Figure~\ref{CompressionRatios}.

\begin{figure}
\centering 
\includegraphics[scale=0.5]{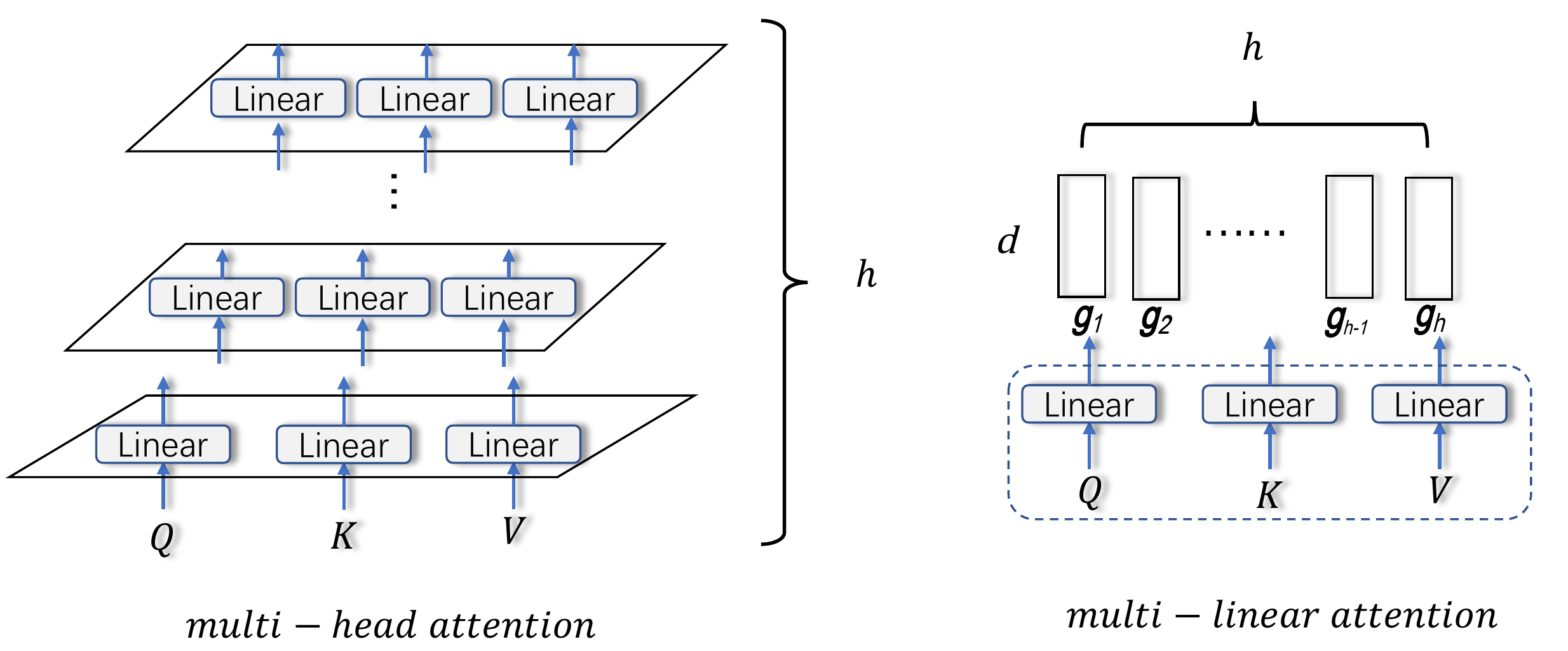}
\caption{A diagram about a comparison of parameters between multi-linear attention and multi-head attention.}
\label{CompressionRatios}
\end{figure}

In Figure~\ref{CompressionRatios}, each $Linear$ function in multi-head attention is about a weight matrix $W \in \mathbb{R}^{d_{model}\times d}$, and all weight matrices in multi-head attention are different. In multi-linear attention, three weight matrices are used and $h$ (a number) weight vectors are used.
Through the analysis about Figure~\ref{CompressionRatios}, the compression ratio is computed as follows.
\begin{equation}
\begin{split}
compression~ratio =&\frac{3\times h \times d_{model} \times d}{3\times d_{model} \times d + h \times d}\\
=&\frac{3\times h \times d_{model}}{3 \times d_{model} + h}
\end{split}
\end{equation}
In practice, $h$ is equal to $8$ and $d$ is equal to $512$. The compression ratio approximates $8$ in this case. In our work, the dimension of vector $\mathcal{G}_r$ is set as $R$ which is smaller than $d$, where $d$ is the dimension of attention matrix. 

\textbf{Low-rank Approximation for Model Compression}
In the paper, we have described that our method combines two compression ideas, namely low-rank approximation and parameters sharing. 
Parameters sharing can be understood through the description of Figure~\ref{CompressionRatios}. 
In Multi-linear attention, the idea of low-rank decomposition also has the function of model compression.
We have proved that the Single-block attention can re-construct an one-head self-attention in Transformer.
In order to obtain the representation of a tensorized attention, we adopt the tensor splitting and the concat function. 
After that, we consider that each tensor slice from tensor splitting approximates the output of the self-attention function~Eq.~2 (in the paper). When we only focus on the idea of low-rank approximation, the compression ratio can be computed by the form, $\frac{N\times d}{N \times N}$, where $N$ is the length of a sequence, $d$ is the dimension of a matrix (also namely hidden size). $N$ is smaller than $d$, normally.

Through combining the ideas of parameters sharing and low-rank approximation, by formally considering the rank $R$, the compression ratio of Multi-linear attention model can be computed as follows:
\begin{equation}
    compression~ ratio_{R} = \frac{3\times h \times d_{model} \times d}{3\times d_{model}\times d + R \times h},
\end{equation}
where $R$ is the rank of the core tensor $\mathcal{G}$. The compression ratio will be larger when $R$ is smaller. This $compression~ratio_R$ is the compression ratio associated with $R$.
 $R$ need to be set in practice. In experiments, $R$ can be set to 
$18$, which is smaller than $d_{model}$.

\section{Experiment}
\subsection{Partial Structure about Tensorized Transformer}
in the paper, the multi-linear attention is proposed. In order to show that the process of incorporating multi-linear attention into Transformer, Figure~\ref{Incorporation} gives out some information about the structure.

\subsection{Experimental Details in Language Modeling}

\begin{figure}
\centering 
\includegraphics[scale=0.5]{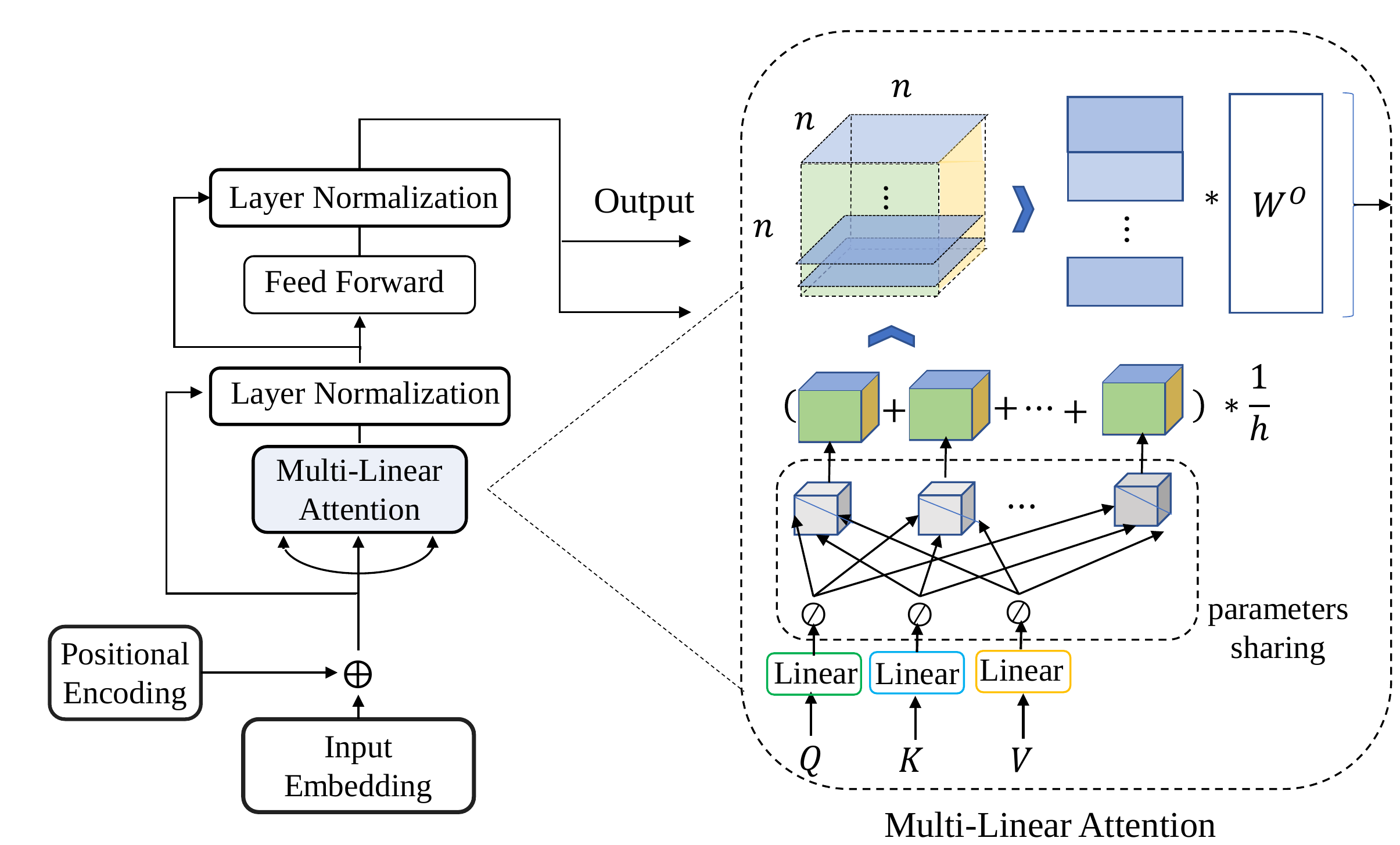}
\caption{A diagram which is about the incorporating of multi-linear attention in partial Transformer structure. The parameters are shared in the constructing of each single-block attention.}
\label{Incorporation}
\end{figure}

%Firstly, we will describe some details about our experiments. Next, we mainly describe the setup and some details. 

Now, we report some details of experiments as a relevant supplementary material. Firstly, we use three weight matrices $W^{q}$,$W^{k}$ and $W^{v}$ to linearly project the queries, keys and values. The outputs from the linear projections can be shared by $h$ times, where $h$ is the number of core tensors in our background (i.e., core-1($h$=$1$), core-2($h$=$2$)). We use Block Term Tensor decomposition (BTD) to construct a new representation, namely Multi-linear attention, which is a $3$-order tensor. For incorporating the proposed attention into the architecture of Transformer, we split the $3$-order tensor, and then concat each matrix from the tensor. For other layers, we use the same structure as vanilla-Transformer.

\textbf{Hardware}

We trained our model on one machine with 2 NVIDIA P40 GPUs. For our base models, the hyperparameters are described in Table~\ref{hyperparameters}. In addition, we set the $dropout$=$0.3$ in all datasets. The model is trained using $30$ epochs in three datasets (PTB, WikiText-103 and One-Billion).
\begin{table}[h]\small
\centering 
\caption{The hyperparameters in the Tensorized Transformers model}
\begin{tabular}{cccccccc}
\toprule[1pt]
Datasets & $d_{head}$ & $d_{ff}$ & h & L& $d_k$ & $d_v$&Test PPL\\
\midrule[0.5pt]
PTB          &256& 2100 & 2 & 3 & 40 &   40  & 49.8\\
WikiText-103 &256& 2100 & 2 & 6 & 40 & 40  &18.9\\
One-Billion  &1024& 2100 & 2 & 6 & 40 & 40 &19.5\\
\bottomrule[1pt]
\end{tabular}
\label{hyperparameters}
\end{table}

\textbf{Optimizer}
We used the Adam optimizer and vary the learning rate over the course of training. The vary formula~\cite{vaswani2017attention} is followed in our work. We also used the $warmup\_steps=4000$. \textbf{Label Smoothing} is employed with the value $\epsilon$=$0.1$.

\subsection{Experiment Details in Neural Machine Translation} 
The Tensorized Transformer also has been applied to Neural Machine Translation task. In this experiment, we use the same setup with Transformer~\cite{vaswani2017attention}, and replace the multi-head attention with the proposed multi-linear attention in the encoder structure. In the decoder structure, we still use the multi-head attention for verifying the effectiveness of encoding a sentence. The model is trained in 1 NVIDA P40 GPUs.

\subsection{Experimental comparison}
For a more detailed comparison, we design these experiments as follows. In this section, we mainly show the experimental results on two language modeling datasets, i.e., PTB and WikiText-103. We show the value of perplexity, as well as FLOPs\footnote{FLOPs:The number of floating-point operations} correspondingly.

\begin{table*}[pthb]\small
\renewcommand\arraystretch{1.2}
  \centering
  \begin{tabular}{llllllll}
    \toprule[1pt]
    \multirow{2}{*}{\textbf{Model}}& \multicolumn{3}{c}{\textbf{PTB}} & \multicolumn{3}{c}{\textbf{WikiText-103}}\\
    \cline{2-7}
    &Params&FLOPS&Test PPL&Params&FLOPS&Test PPL\\
    \hline 
    Transformer-XL~\cite{dai2019transformer} & 24M & 11.5B & 59.1              &257M &996.5B& 18.3\\
    Tensorized Transformer & 24M & 5.4B & 52.7                                 &257M &312.0B& 21.2\\
    \hline 
    Transformer-XL~\cite{dai2019transformer} & -- & --& --                     &151M &126.5B& 24.0\\
    Tensorized Transformer &-- & -- &  --                                      &151M &83.4B& 18.8\\
    \hline 
    Transformer-XL~\cite{dai2019transformer} & 12M & 4.5B & 87.8               &85.5M &22.0B&34.8\\
    Tensorized Transformer &12M & 0.75B &  57.9                                &85.5M &17.3B& 20.9\\

    \bottomrule[1pt]
  \end{tabular}
  \caption{Experimental comparisons on PTB and WikiText-103.}
  \label{AppenT}
  \vspace{-11px}
\end{table*}

\begin{table}[h]\small
\centering 
\caption{The same hyperparameters in Tensorized Transformers and Transformer-XL, $N$ is the length of sequence, and $L$ is the number of layers.}
\begin{tabular}{cccccccc}
\toprule[1pt]
Datasets     & Model                 &$d_{head}$ & $d_{model}$ & $N$ & $L$  & dropout & Test PPL\\
\midrule[0.5pt]
PTB          &Transformer-XL         &40         & 256         & 30 & 3 & 0.3     &81.2\\
PTB          &Tensorized Transformer &40         & 256         & 30 & 3 & 0.3     &50.2\\
WikiText-103 &Transformer-XL         &40         & 256         & 80 & 6 & 0.1     &34.86\\
WikiText-103 &Tensorized Transformer &40         & 256         & 80 & 6 & 0.1     &19.9\\
One-Billion  &Transformer-XL         &40         & 1024        & 100 & 6 & 0.1    &43.6\\
One-Billion  &Tensorized Transformer &40         & 1024        & 100 & 6 & 0.1    &26.7\\

\bottomrule[1pt]
\end{tabular}
\label{samehyper}
\end{table}
To further compare the experimental results under the same size of parameters between Tensorized Transformer and the baseline model (i.e., Transformer-XL), we add some experiments in Table~\ref{AppenT}. In our paper, we use the Tensorized Transformer of 12M on PTB dataset and the Tensorized Transformer of 85.5M on WikiText-103 dataset. To achieve the same size, i.e., 12M for Transformer-XL on PTB, we can reduce the dimensions of $Q,K,V$ from $40$ to $26$. To achieve Transformer-XL of 85.5M, we reduce the dimensions of word embedding from $512$ to $256$. The experimental results are shown in Table~\ref{AppenT}. Our model gets better results and the lower FLOPS than Transformer-XL.

On the other hand, we can also increase the parameters of Tensorized Transformer to reach the parameters reported by Transformer-XL on PTB and WikiText-103 datasets. On PTB dataset, we can increase the number of layers from $3$ to $7$ to get the Tensorized Transformer of 24M. On WikiText-103 dataset, we increase the number of layers from $6$ to $11$ and the length of sequence from $80$ to $120$ to get the Tensorized Transformer of 257M. We can increase the number of layers from $6$ to $8$ and the length of sequence from $80$ to $100$ to get the Tensorized Transformer of 151M.
After that, Tensorized Transformer achieves better results and lower FLOPS than Transformer-XL. These results are shown in Table~\ref{AppenT}.

In addition, we also carry out experiments when Transformer-XL has the same hyperparameters with Tensorized Transformer. Experimental results are shown in Table~\ref{samehyper}. Table~\ref{samehyper} shows that our model can get the better results than Transformer-XL. 
Besides, on two datasets (i.e., PTB and WiliText-103), we also try to train our model (Tensorized Transformer) using larger dimension of word embedding (i.e., $d_{model}$). If $d_{model}$ is larger than $256$ on PTB dataset and larger than $512$ on WikiText-103 dataset, the overfitting will occur. For the overfitting problem, we will investigate it in our future work.

\section{Partial Code}
The project have been achieved by pytorch. In this section, we give the partial code which is about our methods, i.e., Sing-block attention and Multi-linear attention. First, the class of Single-block attention is given as follows.
\begin{lstlisting}
import torch
import torch.nn as nn
import torch.nn.init as init
import numpy as np

class SingleBlockAttention(nn.Module):
    '''Single block attention'''
    def __init__(self, Rank):
        super(SingleBlockAttention, self).__init__()
        self.softmax = nn.Softmax()
        self.R = Rank
    def forward(self, q, k, v, mb_size,d):
        self.core = nn.Parameter(torch.FloatTensor(np.random.rand(self.R)))
        N = v.size(1)
        self.core = self.softmax(self.R)
        core_tensor = torch.zeros(N,d,N).cuda()
        for i in range(self.R):
            cores_tensor[i][i][i] = self.core[i]
        full_matrixs = []
        for i in range(mb_size):
            full_matrix_1 = torch.einsum('pqk, ip,jq,kr->ijr', [core_tensor, q[i], k[i], v[i]]).contiguous()
            full_matrixs.append(torch.sum(full_matrix_1, dim=1))
        output = torch.stack(full_matrixs).cuda().float()
        return output
\end{lstlisting}

Each Single block attention is a component of Multi-linear attention. Based on the Single block attention, the Multi-linear attention can be given as follows.

\begin{lstlisting}
class MultiLinearAttention(nn.Module):
    
    ''' MultiLinearAttention '''

    def __init__(self, h, Rank, d, dropout=0.1):
        super(MultiLinearAttention, self).__init__()
        self.n_head = h # h is equal to 2 in our model
        self.d_k = d
        self.d_v = d
        self.w_q = nn.Parameter(torch.FloatTensor(d_model, d_k))
        self.w_k = nn.Parameter(torch.FloatTensor(d_model, d_k))
        self.w_v = nn.Parameter(torch.FloatTensor(d_model, d_v))
        self.Tattention = SingleCoreAttention(Rank)
        self.layer_norm = LayerNormalization(Rank)
        self.proj = Linear(self.n_head*d, Rank)
        self.dropout = nn.Dropout(dropout)
        init.xavier_normal_(self.w_q)
        init.xavier_normal_(self.w_k)
        init.xavier_normal_(self.w_v)

    def forward(self, q, k, v):

        d_k, d_v = self.d_k, self.d_v
        n_head = self.n_head
        residual = q
        mb_size, len_q, d_model = q.size()
        mb_size, len_k, d_model = k.size()
        mb_size, len_v, d_model = v.size()
        q_s = q.repeat(1, 1).view(-1, d_model)
        k_s = k.repeat(1, 1).view(-1, d_model) 
        v_s = v.repeat(1, 1).view(-1, d_model) 
        if n_head > 1:
          output_1 = self.Tattention(q_s, k_s, v_s, mb_size,d_v)
          output_2 = self.Tattention(q_s, k_s, v_s, mb_size,d_v)
          output = (output_1+output_2)*0.5
        else:
          ouput = self.Tattention(q_s, k_s, v_s, mb_size,d_v)
        # project back to residual size
        outputs = self.proj(outputs)
        outputs = self.dropout(outputs)
        return self.layer_norm(outputs + residual)
        
\end{lstlisting}

\end{document}